\documentclass[final]{l4dc2026}

\title[Provably Safe Stein Variational Clarity-Aware Informative Planning]{Provably Safe Stein Variational Clarity-Aware Informative Planning}
\usepackage{times}

\author{%
 \Name{Kaleb Ben Naveed}\thanks{Equal contribution.} \Email{kbnaveed@umich.edu}\\
 \addr Department of Robotics, University of Michigan, Ann Arbor
 \AND
 \Name{Utkrisht Sahai}\footnotemark[1] \Email{usahai@umich.edu}\\
 \addr Department of Robotics, University of Michigan, Ann Arbor%
 \AND
 \Name{Anouck Girard} \Email{girarda3@erau.edu}\\
 \addr Embry-Riddle Aeronautical University%
 \AND
 \Name{Dimitra Panagou} \Email{dpanagou@umich.edu}\\
 \addr Department of Robotics, Department of Aerospace Engineering, University of Michigan, Ann Arbor%
}
\usepackage[font=scriptsize,labelfont=bf]{caption}

\usepackage{lipsum}
\hypersetup{
  colorlinks,
  citecolor=teal,
  linkcolor=teal,
  urlcolor=purple,
}
\newcommand{\gatekeeper}{\texttt{gatekeeper}}

\newcommand{\reals}{\mathbb{R}}

\newcommand{\R}{\reals}

\newcommand{\Rplus}{\reals_{>0}}

\newcommand{\naturals}{\mathbb{N}}
\renewcommand{\S}{\mathbb{S}}

\newcommand{\pd}{\S_{++}}

\newcommand{\Bcal}{\mathcal{B}}
\newcommand{\Ccal}{\mathcal{C}}

\newcommand{\Ocal}{\mathcal{O}}
\newcommand{\Pcal}{\mathcal{P}}
\newcommand{\Qcal}{\mathcal{Q}}

\newcommand{\Scal}{\mathcal{S}}
\newcommand{\Tcal}{\mathcal{T}}
\newcommand{\Ucal}{\mathcal{U}}

\newcommand{\Xcal}{\mathcal{X}}

\newcommand{\norm}[1]{\left\Vert #1 \right \Vert}

\DeclareMathOperator*{\argmin}{arg\,min}

\usepackage{xcolor, soul}
\definecolor{OliveGreen}{cmyk}{0.17,0,0.41,0.13}
\usepackage{xcolor}
\definecolor{darkgreen}{RGB}{0,100,0} 
\definecolor{tealgreen}{RGB}{0,128,105}

\begin{document}


\maketitle

\begin{abstract}%
Autonomous robots are increasingly deployed for information-gathering tasks in environments that vary across space and time. Planning informative and safe trajectories in such settings is challenging because information decays when regions are not revisited. Most existing planners model information as static or uniformly decaying, ignoring environments where the decay rate varies spatially; those that model non-uniform decay often overlook how it evolves along the robot’s motion, and almost all treat safety as a soft penalty. In this paper, we address these challenges.  We model uncertainty in the environment using clarity, a normalized representation of differential entropy from our earlier work that captures how information improves through new measurements and decays over time when regions are not revisited. Building on this, we present Stein Variational Clarity-Aware Informative Planning, a framework that embeds clarity dynamics within trajectory optimization and enforces safety through a low-level filtering mechanism based on our earlier \gatekeeper{} framework for safety verification. The planner performs Bayesian inference-based learning via Stein variational inference, refining a distribution over informative trajectories while filtering each nominal Stein informative trajectory to ensure safety. Hardware experiments and simulations across environments with varying decay rates and obstacles demonstrate consistent safety and reduced information deficits.
\href{https://usahai18.github.io/stein_clarity/}{[Paper Website]}\footnote{Hardware and simulations experiment videos: \url{https://usahai18.github.io/stein_clarity/}}.
\end{abstract}


\section{Introduction}
Autonomous robots are increasingly used in environmental monitoring~\citep{env_monitoring_2, naveed2025multi}, search and rescue~\citep{search_2}, and 3D reconstruction~\citep{reconstruction_1}, where they must actively gather information to estimate unknown fields such as wind, temperature, or gas concentration. Unlike passive sensing, this requires \textit{informative planning}, a sequential process that determines where and when to sample to reduce uncertainty. The main challenge is to plan trajectories that (i) account for spatially varying information decay rates in the environment, and (ii) remain provably safe with respect to obstacles while exploring informative regions. In this work, we focus on environments characterized by such spatially varying information decay, referred to as \textit{stochastic spatiotemporal environments}.

Existing informative planning approaches, including orienteering-based~\citep{bottarelli2019orienteering}, submodular~\citep{meliou2007nonmyopic}, sampling-based~\citep{moon2025ia}, and GP-based planners~\citep{Chen-RSS-22}, seek to maximize information gain but often assume static fields or model information decay with a constant global parameter. Ergodic exploration~\citep{mathew2011metrics, dressel2019tutorial} offers an alternative by matching the visitation frequency to a target distribution, yet most variants neglect temporal evolution~\citep{seewald2024energy, dong2025time, lee2024stein}. The recent extension in~\citep{naveed2024eclares} addressed spatiotemporal settings but computed the target distribution independently of the robot’s motion, failing to capture uncertainty evolution along its trajectory. Finally, existing methods lack provable safety guarantees, since penalty-based approaches~\citep{lee2024stein} cannot ensure constraint satisfaction, while CBF-based methods~\citep{ames2016control, dong2025time}, although formally safe, remain difficult to design and tune for high-dimensional systems.

We address these challenges by proposing a provably safe, clarity-aware Stein variational informative planning framework. Clarity, introduced in our earlier work~\citep{agrawal2023sensor}, rescales differential entropy to $[0,1]$, making uncertainty easier to interpret and computationally tractable. Using clarity, we explicitly model how information decays and regenerates over time, enabling the robot to reason about the spatiotemporal evolution of uncertainty. We then define a differentiable clarity-based objective and optimize it within a Stein variational inference formulation~\citep{liu2016stein, lambert2020stein}, a particle-based Bayesian learning method that operates in function space to approximate the posterior distribution over trajectories through gradient flows. This formulation allows multiple trajectory candidates to evolve in parallel, enabling the planner to learn and maintain a diverse set of informative motion strategies. To ensure safety, we integrate the \gatekeeper{} framework~\citep{agrawal2024gatekeeper, naveed2025enabling}, where each Stein particle generates a nominal trajectory concatenated with a backup trajectory to form a \textit{candidate trajectory}. A candidate is committed if forward propagation verifies that all states remain within the safe set and the backup terminates inside the backup set. Among safe candidates, the lowest-cost one is executed. As compared to existing methods, our \textbf{contributions} are twofold:
\textcolor{magenta}{\textbf{(i)}} We develop a clarity-aware Stein variational informative planner that learns an approximate posterior distribution over informative trajectories through Bayesian inference, explicitly accounting for spatially varying information decay rates in the environment and the evolution of clarity along the robot’s motion.
\textcolor{magenta}{\textbf{(ii)}} We make the planner provably safe through a low-level \gatekeeper{}-based safety filter that verifies nominal Stein trajectories and ensures that only trajectories certified as safe are committed for execution.

\section{Preliminaries}
\subsection{Notation}
Let $\naturals = \{ 0, 1, 2, ... \} $. Let $\mathbb{R}$, $\mathbb{R}_{\geq 0}$, $\mathbb{R}_{> 0}$ be the set of reals, non-negative reals, and positive reals respectively. Let $\mathcal{S}^{n}_{++}$ denote set of symmetric positive-definite in $\mathbb{R}^{n \times n}$. Let $\mathcal{N}(\mu, \Sigma)$ denote a normal distribution with mean $\mu$ and covariance $\Sigma \in \mathcal{S}^{n}_{++}$. The $Q \in \pd^{n}$ norm of a vector $x \in \R^n$ is denoted $\norm{x}_Q = \sqrt{x^T Q x}.$ Let $\Ccal^i$ denote the space of functions that are $i$ times continuously differentiable with respect to their arguments.

\subsection{Dynamics Model}

Consider the continuous-time robot dynamics
\begin{align}
\dot{x} = f(t,x,u),
\label{eq:ctrlaffine_sys}
\end{align}
with the state \(x \in \Xcal \subset \R^n\), the input \(u \in \Ucal \subset \R^m\), and vector field
\(f:\R \times \Xcal \times \Ucal \to \R^n\).
We assume \(f\) is locally Lipschitz in \(x\) and \(u\) and continuously differentiable in \((x,u)\) on
the domain $\Xcal \times \Ucal$. Given a feedback policy $u = \pi(t,x)$ with $\pi$ continuous in $t$ and locally Lipschitz in $x$, the closed-loop system admits a unique solution over some interval.
\begin{definition}[Trajectory]
\label{def:traj}
Let \(\Tcal=[t_i,t_f]\subset\R\).
A \emph{trajectory} is a pair of functions $\xi = ( \xi_x:\Tcal\to\Xcal, \xi_u:\Tcal\to\Ucal)$, satisfying
$
\dot{\xi
}(t)=f\big(t, \xi
_x(t), \xi
_u(t)\big), \quad \forall t\in(t_i,t_f),
\quad \xi
_x(t_i)=x_i .
$
The set of all trajectories starting from $(t, x) \in \R \times \Xcal$ is denoted
$
\Phi(t, x)
=
\big\{\xi
 = (\xi_x, \xi_u): \xi_x(t)=x \, \text{ and } \,(\xi_x, \xi_u) \text{ is } \text{ a } \text{ trajectory}\}.
$
\end{definition}

Let $\Scal: \R \rightrightarrows \Xcal$ denote the set of states satisfying the constraints. The system satisfies the constraints if $x(t) \in \Scal(t), \forall t \geq t_0$. 
\subsection{Clarity}

We use clarity, introduced in \cite{agrawal2023sensor}, to quantify uncertainty based on differential entropy.

\begin{definition} The \textit{differential entropy} $h[Z] \in (-\infty, \infty)$ of a continuous random variable $Z$ with support $S$ and density $\rho : S \rightarrow \mathbb{R}$ is
\begin{equation}
    h[Z] = -\int_S \rho(z) \log \rho(z) dz.
\end{equation}
\end{definition}
As the uncertainty in $Z$ increases, the entropy approaches $h[Z] \to \infty$. Clarity is defined in terms of differential entropy.
\begin{definition}
Let $Z$ be an $n$-dimensional continuous random variable with differential entropy $h[Z]$. The \textit{clarity} of $Z$ is a normalized quantity $q[Z] \in (0, 1)$ of $Z$ is defined as
\begin{equation}
    q[Z] = \left( 1 + \frac{\exp(2h[Z])}{(2\pi e)^n} \right)^{-1}.
\end{equation}
\end{definition}

In other words, the clarity $q[Z]$ of a random variable $Z$ lies in the interval $(0,1)$, where $q[Z] \to 0$ corresponds to the case where the uncertainty in $Z$ is infinite, and $q[Z] \to 1$ corresponds to the case where $Z$ is perfectly known. See \textit{Example 1} in \cite{agrawal2023sensor}.


\subsection{Stein Variational Trajectory Optimization}

Model Predictive Control (MPC) computes control inputs by solving a finite-horizon optimization problem over a time interval $[t_i, t_f]$ given an initial state $x_i$:
\begin{align}
\min_{(\xi_x, \xi_u) \in \Phi(t_i, x_i)} C(\xi_x, \xi_u),
\label{eq:mpc_opt}
\end{align}
where $C : \Phi(t_i, x_i) \rightarrow \R$ is the trajectory cost functional, and $\Phi(t_i, x_i)$ is the set of admissible trajectories originating from the initial state $x_i$ at time $t_i$ and satisfying the system dynamics and constraints defined in \textit{Def.}~\ref{def:traj}.
Classical MPC returns a single deterministic optimum, which limits its ability to reason over multi-modal solutions that arise from nonconvex costs or constraints.

\paragraph{Bayesian View of Trajectory Optimization.}
Stein Variational Trajectory Optimization reformulates~\eqref{eq:mpc_opt}  
as a Bayesian inference problem over control trajectories $\xi_u$,  
whose corresponding state trajectories $\xi_x$ satisfy $(\xi_x, \xi_u) \in \Phi(t_i, x_i)$.  
Let $\Ocal_{\xi} \in \{0,1\}$ be a binary random variable that indicates whether a control trajectory $\xi_u$ is optimal ($\Ocal_{\xi}=1$) with respect to the cost $C(\xi_x, \xi_u)$.  
By Bayes' rule, the posterior distribution over dynamically feasible control trajectories, conditioned on the current state $x_i$, is given by
\begin{align}
p(\xi_u \mid \Ocal_{\xi}=1, x_i)
=
\frac{p(\Ocal_{\xi}=1 \mid \xi_u, x_i)\, p(\xi_u)}
{p(\Ocal_{\xi}=1 \mid x_i)}.
\label{eq:stein_post}
\end{align}
Here, $p(\xi_u \mid \Ocal_{\xi}=1, x_i)$ is the posterior over trajectories given that they are optimal,  
$p(\Ocal_{\xi}=1 \mid \xi_u, x_i)$ is the likelihood,  
$p(\xi_u)$ is the prior over admissible control trajectories in $\Phi(t_i, x_i)$,  
and $p(\Ocal_{\xi}=1 \mid x_i)$ is the marginal.  
The likelihood is modeled as an exponentiated cost function:
\begin{equation}
p(\Ocal_{\xi}=1 \mid \xi_u, x_i)
\propto
\exp(-\alpha\, C(\xi_x, \xi_u)), \qquad \alpha > 0,
\label{eq:stein_likelihood}
\end{equation}
where $\alpha$ controls the sharpness of the likelihood.  
This formulation assigns higher posterior probability to control trajectories with lower cost,  
while maintaining a belief distribution over the feasible trajectory set.

\paragraph{Variational Approximation.}
The posterior~(\ref{eq:stein_post}) can be approximated by minimizing the Kullback--Leibler (KL) divergence
\begin{equation}
q^\ast = \arg\min_{q \in \Qcal} 
D_{\mathrm{KL}}\!\big(q(\xi_u)\, \|\, p(\xi_u \mid \Ocal_{\xi}=1, x_i)\big),
\label{eq:stein_vi}
\end{equation}
where $\Qcal$ denotes the family of admissible trajectory distributions,
and $q(\xi_u)$ is an empirical distribution represented by a set of $K$ control-trajectory particles 
$\{\xi_u^{(k)}\}_{k=1}^K$.
Each particle $\xi_u^{(k)}$ corresponds to a
control sequence defined over the horizon $[t_i, t_f]$. The state trajectory $\xi_x^{(k)}$ is obtained by integrating the system dynamics~(\ref{eq:ctrlaffine_sys})
from the current state $x_i$ under the control sequence $\xi_u^{(k)}$.

\paragraph{Stein Variational Gradient Update.}
Stein Variational Gradient Descent (SVGD)~\citep{liu2016stein} 
deterministically transports a set of particles to minimize the KL divergence 
between their empirical distribution and the target posterior.
Each particle here represents a feasible control trajectory $\xi_u^{(k)}$, 
which is iteratively updated according to
\begin{equation}
\xi_u^{(k)} \leftarrow \xi_u^{(k)} + \varepsilon\, \phi^\ast(\xi_u^{(k)}),
\end{equation}
where $\varepsilon > 0$ is the step size.
The optimal update field $\phi^\ast(\cdot)$ defines the direction that maximally decreases the KL divergence and is given by
\begin{equation}
\phi^\ast(\xi_u^{(k)}) =
\frac{1}{K}
\sum_{j=1}^K
\Big[
k(\xi_u^{(j)},\xi_u^{(k)})\, 
\nabla_{\xi_u^{(j)}} \log p(\xi_u^{(j)} \mid x_i)
+ 
\nabla_{\xi_u^{(j)}} k(\xi_u^{(j)},\xi_u^{(k)})
\Big],
\label{eq:stein_svgd}
\end{equation}
where $k(\cdot,\cdot)$ is a positive-definite kernel defining the Reproducing Kernel Hilbert Space (RKHS) used to compute $\phi^\ast$. In the first term, the kernel weights each particle’s gradient by its similarity to others, pulling them toward high-posterior regions; the second term introduces a repulsive force that preserves diversity in the RKHS.
The log-posterior gradient is computed as
\begin{equation}
\nabla_{\xi_u^{(j)}} \log p(\xi_u^{(j)} \mid x_i)
=
-\alpha\, \nabla_{\xi_u^{(j)}} C(\xi_x^{(j)}, \xi_u^{(j)})
+ \nabla_{\xi_u^{(j)}} \log p(\xi_u^{(j)}),
\label{eq:stein_grad}
\end{equation}
where the first term depends on the trajectory cost and the second encodes any prior belief over control trajectories 
(often taken uniform when no prior preference is imposed).

\section{Problem Formulation}

\subsection{Environment Specification}
Consider the coverage space $\Pcal$. We discretize the domain into a set of $N_p$ cells each with size $V$.\footnote{Size is length in 1D, area in 2D, and volume in 3D.} Let $m_p: [t_0, \infty) \to \mathbb{R}$ be the (time-varying) quantity of interest at each cell $p \in \Pcal_{\text{cells}} = \{1, ..., N_p\}$. We model the quantities of interest as independent stochastic processes:
\begin{subequations}
\begin{align}
    \Dot{m}_p &= w_{p}(t), &&w_{p}(t) \sim \mathcal{N} (0, Q_{p}), \label{eqn: quantity_of_interest_process} \\
    y_p &= C_{p}(x) m_{p} + v_p(t),  &&v_p(t) \sim \mathcal{N} (0, R), \label{eqn: quantity_of_interest_measurement}
\end{align}
\label{eqn: quantity_of_interest}
\end{subequations}
where $y_p \in \R$ is the output corresponding to cell $p$. 
$R$ is the measurement noise variance, and $Q_p \in \Rplus$ is the process noise variance at each cell $p$.
Since the process noise $Q_p$ varies across cells, the field values $m_p$ evolve differently in time at each location, 
resulting in a \emph{stochastic spatiotemporal environment}. To quantify the uncertainty in the environment, we define an independent clarity dynamics for each cell $p \in \Pcal_{\text{cells}}$. The clarity dynamics derived in \cite{agrawal2023sensor} for each cell are given as follows:

\begin{equation}
\begin{aligned}
\dot{q_p}= \frac{C_p(x)^2}{R}(1 - q_p)^2 - Q_pq_p^2
\label{eqn:clarity_dynamics}
\end{aligned}
\end{equation}

where $x$ is the state of the robot, $C_p: \Xcal \to \R$ is the mapping between robot state and sensor state at cell $p \in \Pcal$, and $R \in \R$ is the known variance of the measurement noise. Here, $q_p \rightarrow 1$ represents the case when the state of the environment (e.g. smoke concentration) is perfectly known in the cell $p$, whereas lower values correspond to higher uncertainty.

\subsection{Problem Statement}
Consider a robotic system evolving under dynamics~(\ref{eq:ctrlaffine_sys}) within a stochastic spatiotemporal environment~(\ref{eqn: quantity_of_interest}).
Each cell $p \in \Pcal$ evolves with process noise $Q_p$, and its clarity state follows~(\ref{eqn:clarity_dynamics}).
A desired information level is specified as a \emph{target clarity} $\overline{q}_p < q_{\infty,p}$, where $q_{\infty,p}$ is the maximum attainable clarity, ensuring that the target is achievable within finite time.
Clarity is used as the information metric because it naturally reflects both environmental stochasticity and sensor characteristics: (i) the decay term $-Q_p q_p^2$ links the information decay rate directly to the process noise variance $Q_p$, and (ii) during sensing, $q_p$ increases monotonically and saturates at $q_{\infty,p} < 1$ for $Q_p, R(x) > 0$, guaranteeing an upper bound on attainable information.

The robot generates the informative trajectory in a receding-horizon manner: 
at each time $t_i$, it optimizes a feasible trajectory 
$\xi = (\xi_x, \xi_u) \in \Phi(t_i, x_i)$ over a finite horizon 
$[t_i, t_i + T_N]$, executes the first control input, and replans at the next step. The optimization problem is formulated as:
\begin{subequations}
\label{eq:clarity_opt}
\begin{align}
\min_{\xi \in \Phi(t_i, x_i)} \quad &
\frac{1}{N_p} \sum_{p=1}^{N_p} 
\int_{t_i}^{t_i + T_N}\max\big(0, \, \overline{q}_p - q_p(\tau)\big)d\tau
\label{eq:clarity_opt_obj}\\
\text{s.t.} \quad 
& \dot{\xi}_x = f(t, \xi_x(t), \xi_u(t)), \quad \forall t \in [t_i, t_i + T_N], \label{eq:clarity_opt_dyn}\\
& \dot{q}_p = g(\xi_x, q_p), \quad \forall p \in \Pcal_{\text{cells}}, \label{eq:clarity_opt_clarity}\\
& \xi_x(t) \in \Scal(t), \quad \forall t \in [t_i, t_i + T_N], \label{eq:clarity_opt_safe} \\
& \xi_x(t_i) = x_i \label{eq:dyn_init}
\end{align}
\end{subequations}
where $g(\cdot,\cdot)$ denotes the clarity dynamics~(\ref{eqn:clarity_dynamics}), 
and $\Scal(t) \subset \Xcal$ defines the time-varying safe set. 
Problem~(\ref{eq:clarity_opt}) seeks a feasible trajectory $\xi = (\xi_x, \xi_u)$ that respects system dynamics, clarity evolution, and safety constraints while minimizing the mean clarity deficit across the environment.




\section{Methodology}
\begin{figure*}[t]
  \centering
\includegraphics[width=1.0\columnwidth]{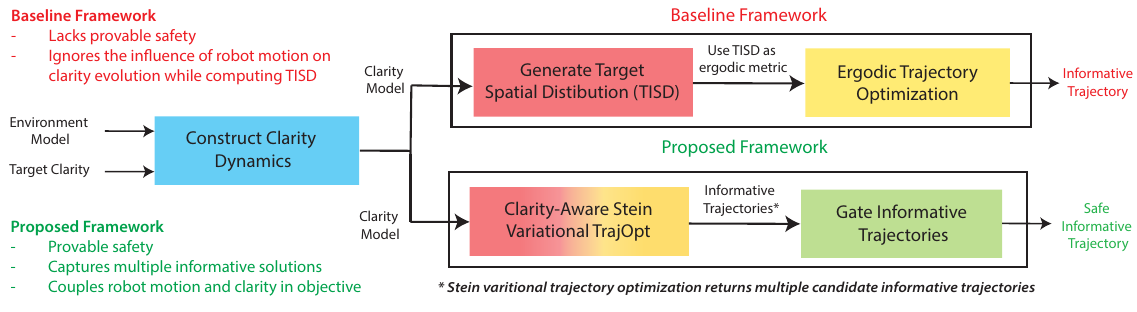}
  \caption{The baseline derives a TISD from the clarity model and uses it as an ergodic metric, whereas the proposed framework directly optimizes the clarity model via a Stein variational approach, producing multiple informative and provably safe trajectories.}
  \vspace{-10pt}
  \label{fig:block_diag}
\end{figure*}

\subsection{Overview}

The proposed framework, illustrated in Fig.~\ref{fig:block_diag}, builds upon clarity-aware ergodic search by removing the multi-stage approximation introduced by the Target Information Spatial Distribution (TISD).
In the baseline approach, the TISD for each cell is computed as the time required for its clarity to reach the target value, assuming the robot is co-located with that cell.
This neglects the fact that the robot can only influence clarity along its actual trajectory, breaking the coupling between motion and clarity evolution and yielding a static spatial objective that fails to reflect how information evolves during motion.

To address this, the proposed method directly couples the robot trajectory with the clarity dynamics through a Stein variational formulation.
Multiple clarity-aware trajectories are evolved in parallel by minimizing a differentiable clarity-based cost that captures the spatiotemporal evolution of information. Safety is enforced by gating the nominal Stein trajectories, where each is paired with a backup trajectory, forward propagated, and evaluated for safety over the entire horizon.
The safe candidate with the lowest cost is committed for execution.

\subsection{Clarity-Aware Stein Variational Trajectory Optimization}

The proposed framework performs trajectory optimization using a differentiable clarity-based cost. 
Since the Stein variational update requires gradients of the objective with respect to the trajectory, 
both the cost functional and the measurement model must be continuously differentiable.

\paragraph{Differentiable Clarity-Aware Cost Functional.}
The original mean clarity deficit objective~\eqref{eq:clarity_opt_obj} contains a non-differentiable hinge term 
$\max(0, \overline{q}_p - q_p)$, which prevents analytic gradient computation through the clarity model. 
We replace this hinge with a smooth \emph{softplus} surrogate:
\begin{equation}
\text{softplus}_\beta(z) = \frac{1}{\beta}\log\!\big(1 + e^{\beta z}\big), \qquad \beta > 0,
\label{eq:softplus}
\end{equation}
where larger $\beta$ values yield a closer approximation to $\max(0, z)$ while maintaining differentiability.  The resulting differentiable clarity-aware cost for a trajectory 
$\xi = (\xi_x, \xi_u) \in \Phi(t, x)$ over the horizon $[t_i, t_i + T_N]$ is
\begin{equation}
J_\beta(\xi)
=
\frac{1}{N_pT_N}
\sum_{p=1}^{N_p}
\int_{t_i}^{t_i + T_N}
\text{softplus}_\beta\Big(\overline{q}_p - q_p(\tau; \xi_x)\Big)
\, d\tau,
\label{eq:clarity_cost_soft}
\end{equation}
where $q_p(\tau; \xi_x)$ denotes the clarity at cell $p$ obtained by integrating the clarity dynamics~\eqref{eqn:clarity_dynamics} along the trajectory $\xi_x$.
The cost $J_\beta(\xi)$ is continuously differentiable in $\xi$ for any finite $\beta > 0$, 
allowing smooth gradient propagation through the clarity model and trajectory.

\paragraph{Differentiable Measurement Model.}
The measurement mapping $C(x)$ determines how the robot’s state influences clarity gain in each cell. 
To preserve differentiability of $J_\beta(\xi)$, $C(x)$ must be at least $\Ccal^1$ and locally Lipschitz in $x$. 
We model each cell-specific sensing footprint as a Gaussian field:
\begin{equation}
\label{eq:C_gaussian}
C_p(x)
=
\kappa_c \exp\Big(-\frac{1}{2}\|x - \mu_p\|^2_{\Sigma_c^{-1}}\Big),
\end{equation}
where $\mu_p \in \R^n$ is the center of cell $p$, $\Sigma_c \in \mathcal{S}^{n}_{++}$ defines the sensing footprint, 
and $\kappa_c > 0$ scales the sensor strength. 
This ensures that $C_p(x)$ and its gradient $\nabla_x C_p(x)$ are continuous and bounded, 
enabling differentiable backpropagation through the clarity dynamics~\eqref{eqn:clarity_dynamics}. 
Alternative smooth sensing models (e.g., radial basis or polynomial decay) are also admissible 
as long as $C_p(x)$ remains continuously differentiable.

\subsection{Safety filtering through \gatekeeper{} framework}

To ensure safety, we propose a method based on the \gatekeeper{} framework~\citep{agrawal2024gatekeeper} 
to filter and commit safe trajectories among the candidates generated by the 
Stein variational optimizer. 
At each decision time $t_i$, the optimizer produces $K$ nominal trajectories, 
each evolved over a nominal horizon $[t_i,\, t_i + T_N]$. 
Along each nominal trajectory, multiple switching times are defined; at each switching time, a candidate trajectory is constructed by concatenating a backup trajectory to the nominal segment.
All candidates are evaluated in parallel for safety and cost, 
ensuring that the executed trajectory remains within $\Scal(t)$ for all $t \ge t_0$. 
Figure~\ref{fig:gating_concept} illustrates the procedure for $K=4$.

\begin{definition}[Nominal Trajectory]
\label{def:nominal_traj}
At time $t_i$, the Stein variational optimizer generates a set of nominal trajectories
$\Xi = \{\xi^{(k)}\}_{k=1}^K$, where each 
$\xi^{(k)} = (\xi_x^{(k)}, \xi_u^{(k)}) \in \Phi(t_i, x_i)$
is a trajectory evolved under the system dynamics~\eqref{eq:ctrlaffine_sys} 
over the nominal horizon $[t_i,\, t_i + T_N]$. 
\end{definition}

\begin{figure*}[t]
  \centering
\includegraphics[width=1.0\columnwidth]{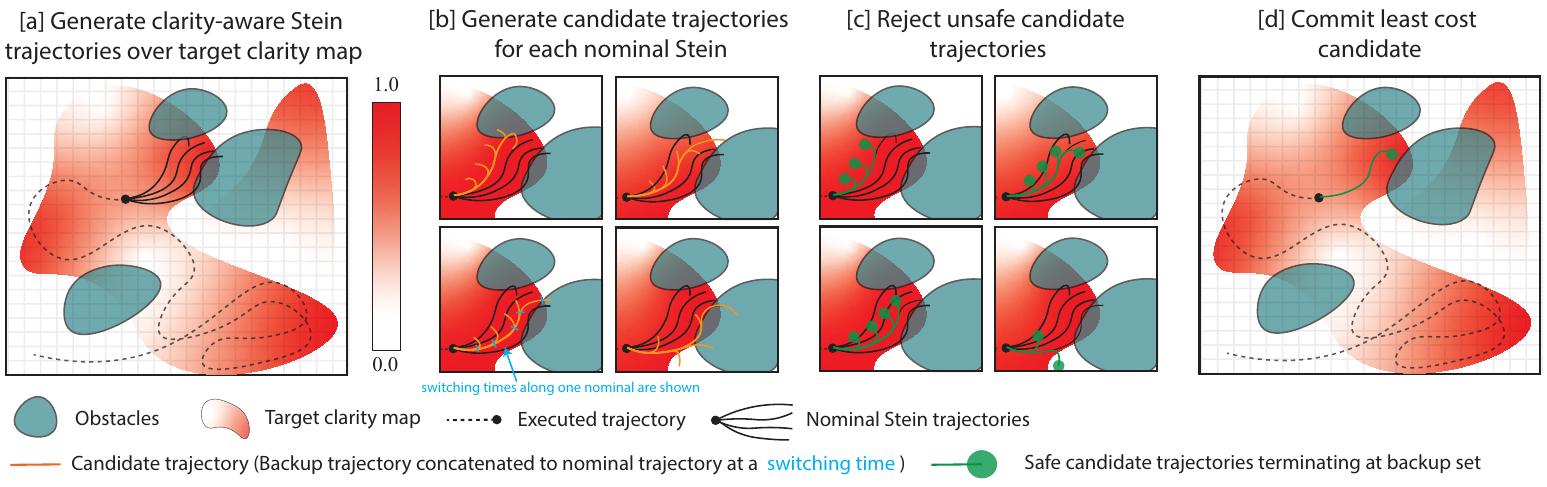}
  \caption{Safety filtering through the \gatekeeper{} framework. Candidate trajectories (shown in yellow) are generated in parallel for each nominal Stein trajectory. Unsafe candidate trajectories are rejected, and the least-cost safe candidate is executed.}
  \vspace{-10pt}
  \label{fig:gating_concept}
\end{figure*}
\begin{definition}[Backup Set]
\label{def:backup_set}
The \emph{backup set}, denoted $\Bcal(t) \subseteq \Scal(t)$, 
is a time-varying subset of the safe set such that there exists a 
\emph{backup controller} $\pi^\text{b}: \R_{\ge 0} \times \Xcal \to \Ucal$ 
for which the closed-loop dynamics 
$\dot{x} = f(t,x, \pi^\text{b}(t, x))$
satisfy the forward-invariance condition
\begin{equation}
x(t_0) \in \Bcal(t_0) \;\Rightarrow\; x(t) \in \Bcal(t), 
\quad \forall\, t \ge t_0.
\end{equation}
Hence, $\Bcal(t)$ is the region from which safety is guaranteed under $\pi^\text{b}$.
\end{definition}

\begin{definition}[Valid Backup Trajectory]
\label{def:backup_traj}
Given a nominal trajectory $\xi^{(k)}$ and a switching time $\tau_s \in [t_i,\, t_i + T_N]$, 
a \emph{backup trajectory} 
$\xi^\text{b} = (\xi_x^\text{b}, \xi_u^\text{b})$
is a feasible trajectory initialized from the nominal state at the switching time,
$
\xi_x^\text{b}(\tau_s) = \xi_x^{(k)}(\tau_s),
$
and evolved over a backup horizon $T_B$. 
The trajectory $\xi^\text{b}$ is \emph{valid} if it remains within the safe set and 
terminates in the backup set (\textit{Def}~\ref{def:backup_set}):
\begin{equation}
\xi_x^\text{b}(t) \in \Scal(t),
\quad \forall\, t \in [\tau_s,\, \tau_s + T_B],
\qquad
\xi_x^\text{b}(\tau_s + T_B) \in \Bcal(\tau_s + T_B).
\end{equation}
\end{definition}
\begin{definition}[Candidate Trajectory]
\label{def:candidate}
Given a nominal trajectory $\xi^{(k)} \in \Phi(t_i, x_i)$ 
and a switching time $\tau_s \in [t_i,\, t_i + T_N]$, 
a \emph{candidate trajectory} $\xi^{c_{(k,s)}} \in \Phi(t_i, x_i)$ 
is defined over the horizon $[t_i,\, \tau_s + T_B]$ as the concatenation of 
(1) the nominal segment, initialized at $x(t_i) = x_i$ and propagated under 
$\xi^{(k)}$ up to $\tau_s$, and 
(2) a backup trajectory $\xi^\text{b}$, initialized at 
$\xi_x^\text{b}(\tau_s) = \xi_x^{(k)}(\tau_s)$ and generated for a backup horizon $T_B$:
\begin{equation}
\xi^{c_{(k,s)}} = 
\big[\xi^{(k)}(t),\, t \in [t_i,\, \tau_s]\big] \oplus 
\big[\xi^\text{b}(t),\, t \in [\tau_s,\, \tau_s + T_B]\big].
\end{equation}
\end{definition}
\begin{definition}[Safe Candidate Trajectory]
\label{def:safe_candidate}
A candidate trajectory $\xi^{c_{(k,s)}}$ is \emph{safe} if all its states remain within 
the safe set and its concatenated backup is valid by \textit{Def}~\ref{def:backup_traj}:
\begin{equation}
\xi_x^{c_{(k,s)}}(t) \in \Scal(t),
\quad \forall\, t \in [t_i,\, \tau_s + T_B],
\quad \text{and} \quad
\xi^\text{b} \text{ is valid by \textit{Def}~\ref{def:backup_traj}}
\end{equation}
Let $\Phi_\text{safe}(t_i,x_i)$ denote the set of all safe candidates.
\end{definition}

\begin{definition}[Committed Trajectory]
\label{def:commit}
The least-cost safe candidate is committed for execution:
\begin{equation}
\xi^\star = 
\argmin_{\xi^{c_{(k,s)}} \in \Phi_\text{safe}(t_i,x_i)} 
J_\beta\big(\xi^{c_{(k,s)}}\big).
\end{equation}
If no safe candidate exists, 
the robot continues to follow the previously committed trajectory.
\end{definition}

\begin{theorem}
Assuming that at time $t_0$ a committed trajectory is available, and that at subsequent 
planning times $\{t_0, t_1, \dots, t_i, \dots\}$ the committed trajectories are 
constructed according to Def.~\ref{def:commit}, 
the closed-loop state remains safe for all future time, i.e.,
\begin{equation}
x(t) \in S(t), \quad \forall\, t \ge t_0. 
\end{equation}
\label{thm:global_safety}
\end{theorem}


\begin{proof}
The proof is provided in Appendix~\ref{appendix:proof_global_safety}.
\end{proof}

\section{Results \& Discussion}
\begin{figure*}[t]
  \centering
\includegraphics[width=1.0\columnwidth]{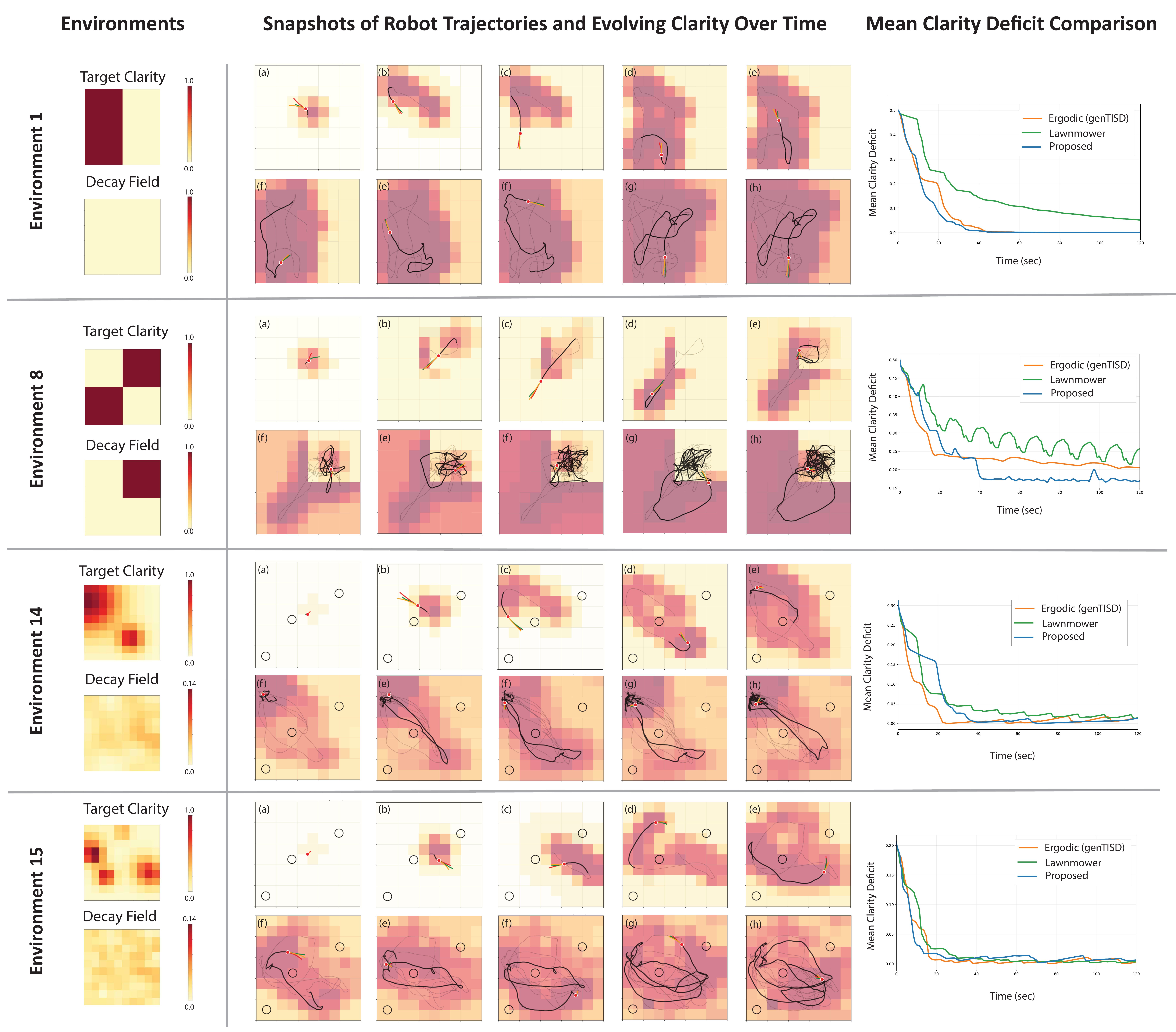}
  \caption{Planner behavior across representative environments. The planner maintains low clarity deficits and safe motion.}
  \vspace{-10pt}
  \label{fig:main_results}
\end{figure*}

To evaluate the proposed framework, we set up sixteen simulation environments covering different combinations of target clarity fields, information decay, and obstacles. These environments test the planner’s ability to explore informative regions, maintain clarity over time, and remain safe. We present four representative cases that illustrate the key behaviors observed across all environments. All sixteen environments are shown in Appendix~\ref{appendix:all_results}. We also demonstrate the proposed method through hardware experiments, as shown in Appendix~\ref{appendix:hardware_exps}.

\subsection{Planner Behavior Across Environments}
\paragraph{Environment~1:}
This environment has a static field with no information decay, as the process noise is zero. As shown in Fig.~\ref{fig:main_results}, the left half of the grid has higher target clarity, while the right half is zero. The robot first moves to the high target clarity region and increases clarity until the target is reached. Once the clarity deficit in that region approaches zero, the objective provides no further incentive to remain there, allowing the robot to move freely across the grid. Since there is no decay, it does not need to revisit previously explored areas.
\begin{table}[t!]
\centering
\scriptsize
\caption{Comparison of Methods Across Experiments}
\begin{tabular}{ccc}
\hline
\textbf{Method} & \textbf{No. of Experiments} & \textbf{Mean Safety Violations (\%)} \\
\hline
Proposed (\textcolor{red}{without gating}) & 100 across 4 envs & 2.94 \\
Proposed (\textcolor{darkgreen}{with gating}) & 100 across 4 envs & 0.0 \\
\hline
\end{tabular}
\label{tab:comparison_methods}
\end{table}

\paragraph{Environment~8:}
This environment contains both decaying and non-decaying regions, as shown in Fig.~\ref{fig:main_results}. The top-right and bottom-left quadrants have a target clarity of 1, while the others are 0. Decay is present only in the top-right region. The robot first moves to the bottom-left quadrant, reaching the target clarity quickly since clarity there does not decay. It then moves to the top-right, where decay causes clarity to decrease over time. The robot stays in this region to minimize the clarity deficit, illustrating how the planner adapts to temporal decay.

\paragraph{Environments~14 and~15:}
These environments contain obstacles near high-clarity regions (Figure~\ref{fig:main_results}) and test the planner’s ability to explore informative areas safely. The robot adapts its motion to the clarity deficit, reaching high-clarity regions while avoiding collisions. The integrated safety mechanism keeps all trajectories within the safe set, maintaining safety as the planner minimizes clarity deficit. Additional results are provided in Appendix~\ref{appendix:all_results}.

\subsection{Comparison of Mean Clarity Deficit and Safety Verification}

Figure~\ref{fig:main_results} compares the mean clarity deficit across methods, while table~\ref{tab:comparison_methods} reports safety performance. The lawnmower strategy performs worst, following a fixed sweeping pattern that cannot adapt to high clarity-deficit regions. Both the ergodic and Stein variational planners achieve lower deficits by redistributing coverage as the field evolves. Their information-gathering performance is comparable; however, the Stein approach is far more efficient. Generating Stein trajectories over a 6\,s horizon takes only 70\,ms with particle-based parallelization in JAX, while the \gatekeeper{} safety check adds about 40\,ms. In contrast, computing an ergodic trajectory of the same horizon via nonlinear optimization takes about 0.7\,s.

For the safety evaluation in table \ref{tab:comparison_methods}, we used four environments with obstacles and ran 100 randomized trials varying the robot’s initial position and obstacle size and location. Without the \gatekeeper{}, soft-penalty enforcement led to an average of 2.94\,\% safety violations, denoting the fraction of time in collision. With the \gatekeeper{} enabled, all runs were collision-free (0.0\,\%), confirming constraint satisfaction without loss of real-time performance.

\section{Conclusion}
We proposed a Stein variational clarity-aware informative planner that integrates Bayesian learning of informative trajectories with safety verification. Clarity, derived from differential entropy, provides a normalized representation of information quality that varies across space and time. Its dynamics capture how uncertainty decays as the robot moves through the environment. The Stein variational formulation learns diverse trajectory distributions by coupling motion with clarity dynamics, while the \gatekeeper{} ensures safety through real-time filtering. Simulations demonstrate safe operation, computational efficiency, and reduced information deficits. Future work will extend the framework to multi-robot settings and online adaptation under unknown process noise.

\acks{This work has been partially supported by the National Science Foundation (NSF) under Award Numbers 1942907 and 2223845.}

\bibliography{l4dc2026-sample}

\newpage
\appendix
\section{Hardware Experiments}
\label{appendix:hardware_exps}

\begin{figure*}[t]
  \centering
\includegraphics[width=1.0\columnwidth]{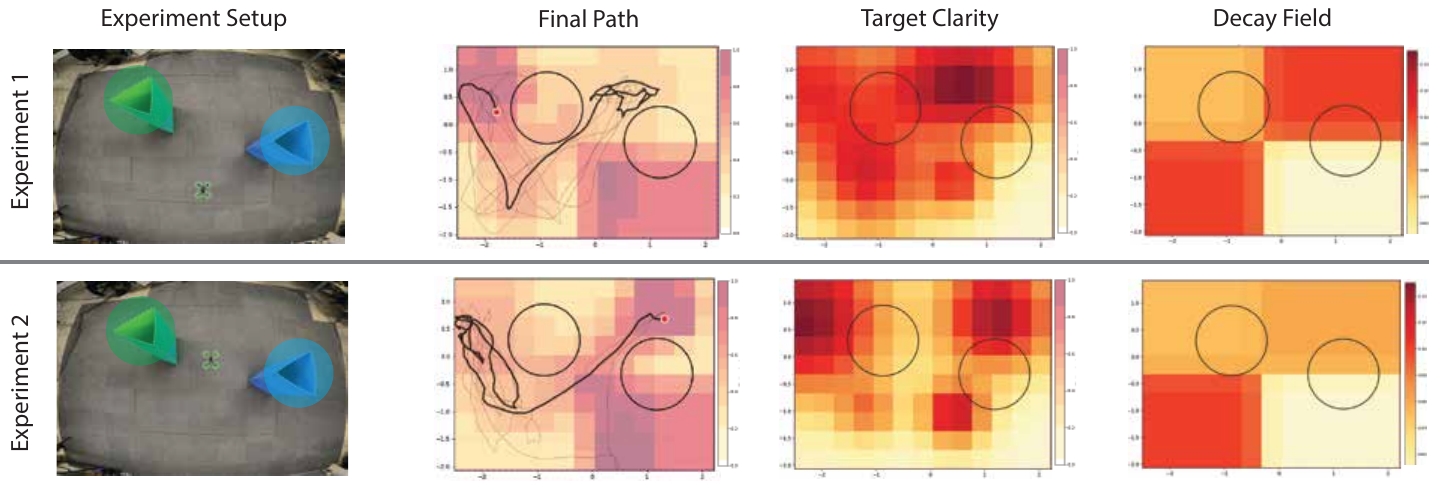}
  \caption{Hardware Experiments}
  \vspace{-10pt}
  \label{fig:Exp_hardware}
\end{figure*}

Two environments with different configurations of the target clarity and decay field were demonstrated. Experiments were conducted in an indoor arena with 17 Vicon cameras for state estimation. During planning, the physical obstacles were modeled as circular objects with appropriate padding. Stein trajectories with a 6\,s horizon were generated in a receding-horizon manner for the double-integrator model and tracked by the quadrotor using differential flatness. The average computation time per replanning step, including Stein trajectory generation and safety verification, was 100\,ms. The experiment videos can be found here:
\href{https://usahai18.github.io/stein_clarity/}{[Experiment Videos]}\footnote{Experiment Videos: \url{https://usahai18.github.io/stein_clarity/}}.

\section{Proof of Theorem~\ref{thm:global_safety}}
\label{appendix:proof_global_safety}
\begin{proof}
We prove the result by induction over the sequence of planning times 
$t_0, t_1, \dots, t_k, \dots$. 

\emph{Base case ($t_0$):}
By assumption, a committed trajectory $\xi_0^{\mathrm{com}}$ is available at time $t_0$.
Since $\xi_0^{\mathrm{com}} \in \Phi_{\mathrm{safe}}(t_0, x(t_0))$ by construction,
its numerically integrated state satisfies $x(t) \in S(t)$ for all $t \in [t_0, t_1]$.
Hence the system is safe during the first execution interval.

\emph{Inductive step:}
Assume that for some $k \ge 0$, the system remains safe over $[t_k, t_{k+1}]$
and that a committed trajectory $\xi_k^{\mathrm{com}}$ is available at $t_k$.
At time $t_{k+1}$, the next committed trajectory $\xi_{k+1}^{\mathrm{com}}$ 
is constructed according to Def.~\ref{def:commit}. 

If $\Phi_{\mathrm{safe}}(t_{k+1}, x(t_{k+1})) \neq \emptyset$,
then $\xi_{k+1}^{\mathrm{com}}$ is a validated safe trajectory whose state 
remains in $S(t)$ for all $t \in [t_{k+1}, t_{k+2}]$.
Otherwise, by the recursive property of Def.~\ref{def:commit}, 
the system continues executing the previously committed trajectory 
$\xi_k^{\mathrm{com}}$, which has already been validated to remain safe 
over its remaining horizon.
In both cases, $x(t) \in S(t)$ for all $t \in [t_{k+1}, t_{k+2}]$.

By induction, $x(t) \in S(t)$ for all $t \ge t_0$, 
establishing the global safety of the closed-loop system.
\end{proof}

\section{More Simulation Results}
\label{appendix:all_results}
\begin{figure*}[t]
  \centering
\includegraphics[width=1.0\columnwidth]{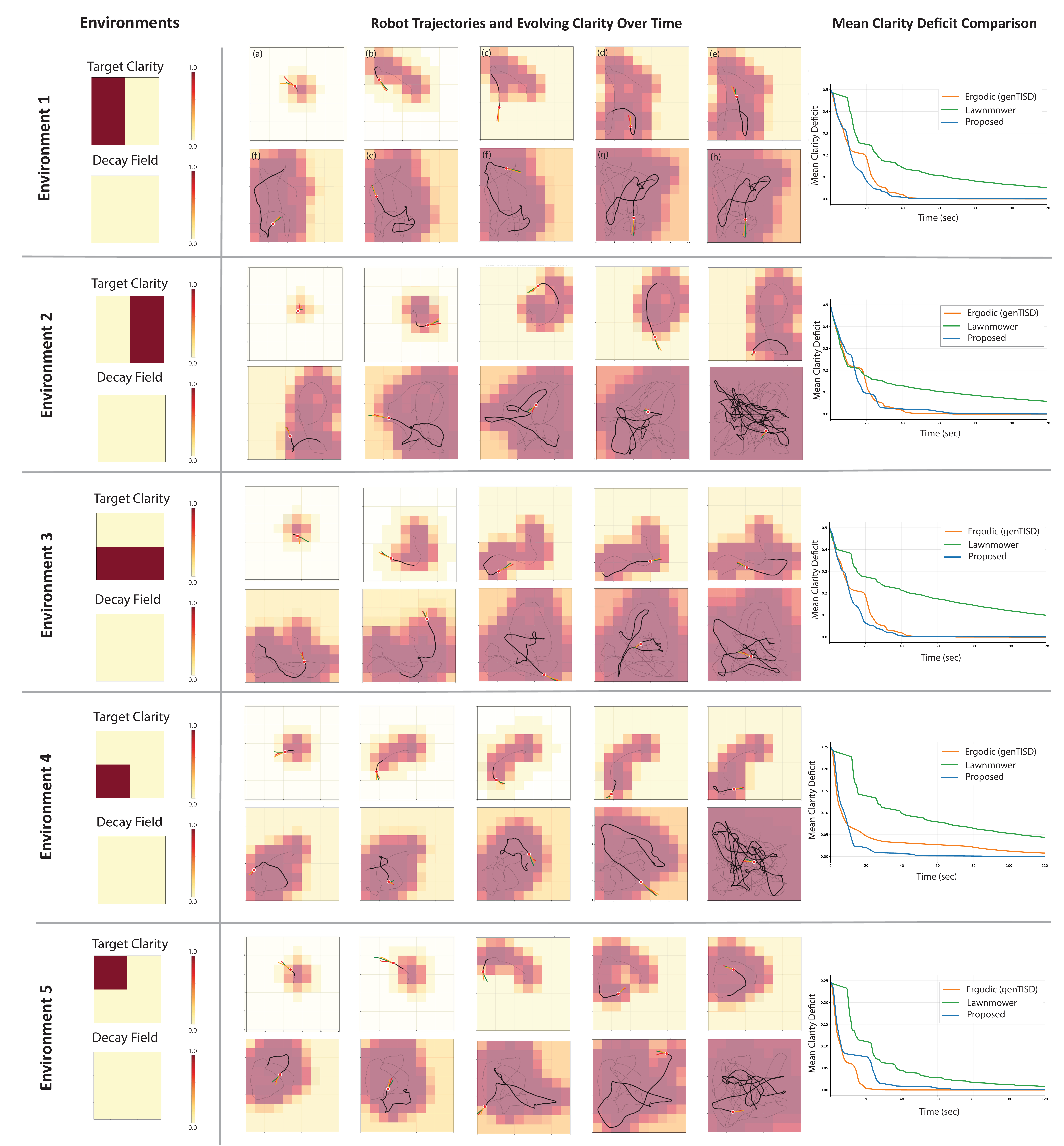}
  \caption{}
  \vspace{-10pt}
  \label{fig:Exp1_5}
\end{figure*}

\begin{figure*}[t]
  \centering
\includegraphics[width=1.0\columnwidth]{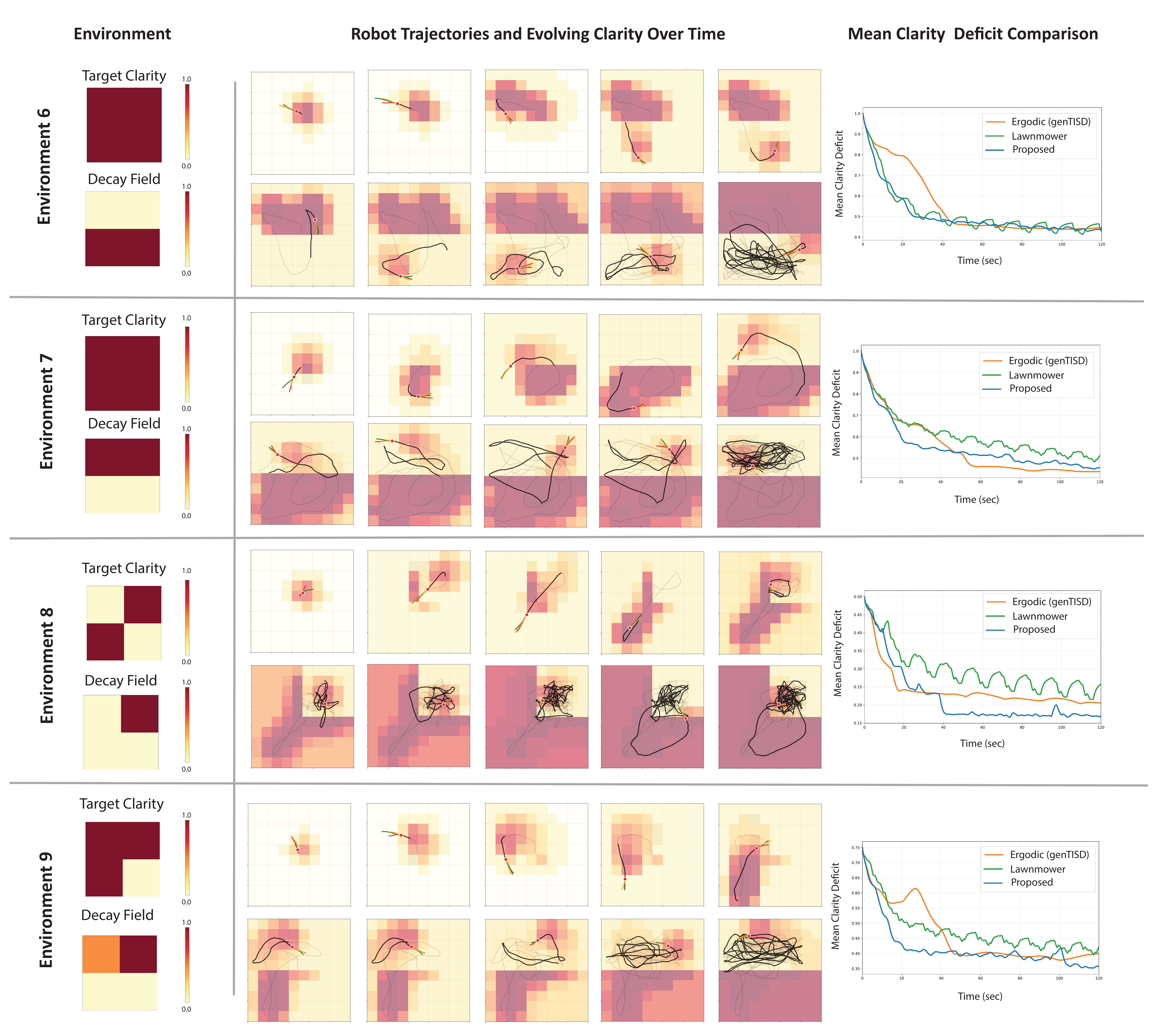}
  \caption{}
  \vspace{-10pt}
  \label{fig:Exp6_9}
\end{figure*}

\begin{figure*}[t]
  \centering
\includegraphics[width=1.0\columnwidth]{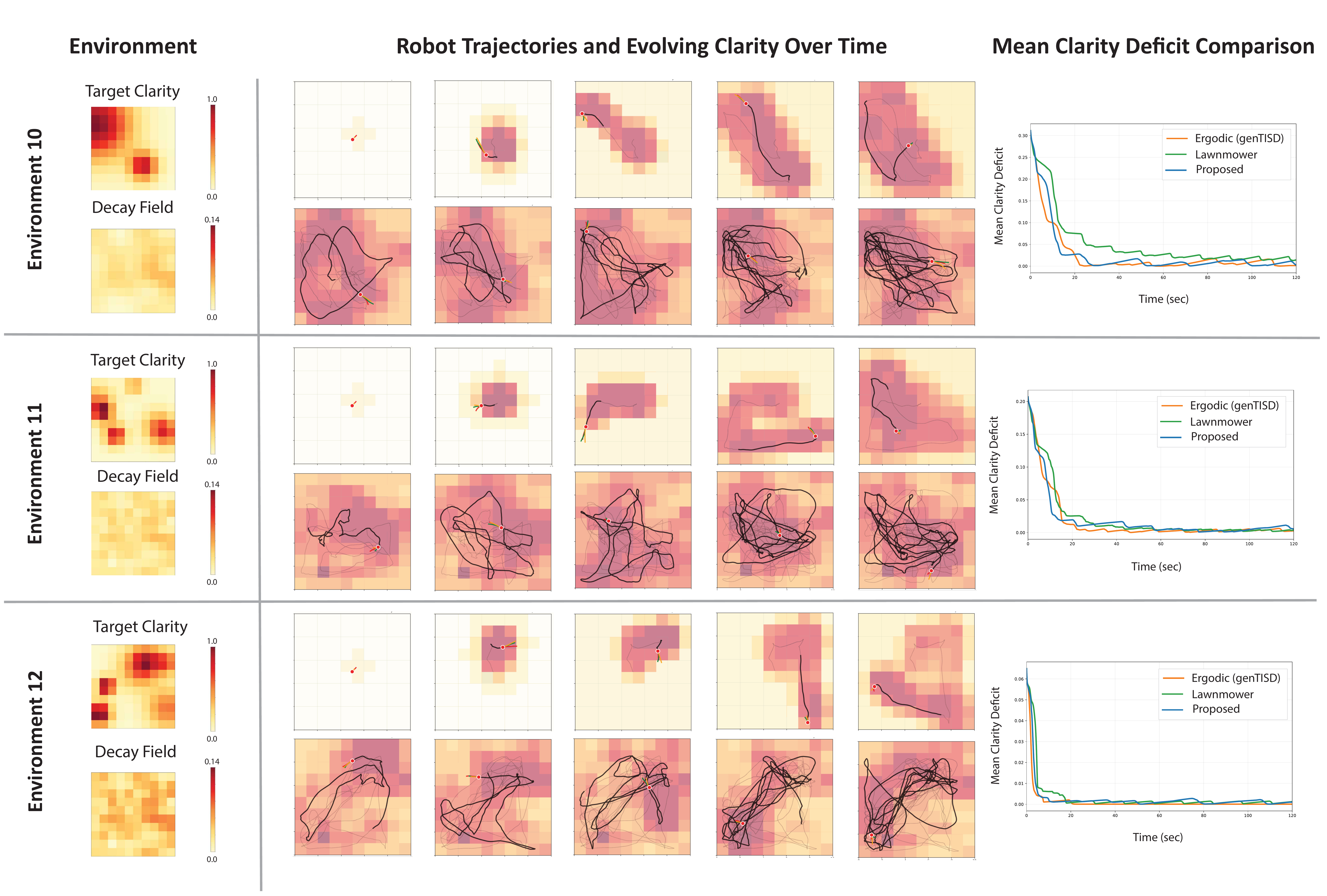}
  \caption{}
  \vspace{-10pt}
  \label{fig:Exp10_12}
\end{figure*}

\begin{figure*}[t]
  \centering
\includegraphics[width=1.0\columnwidth]{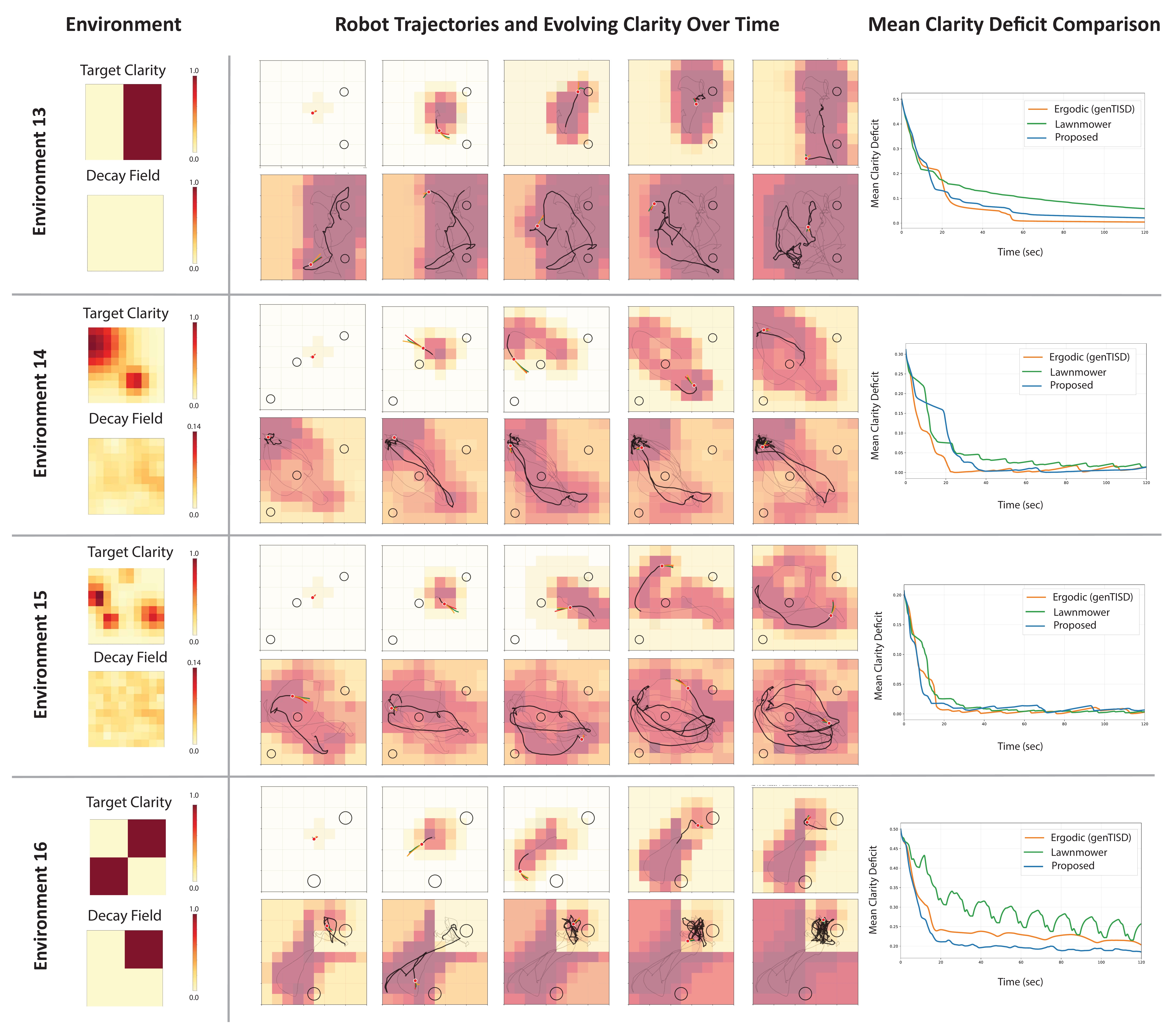}
  \caption{}
  \vspace{-10pt}
  \label{fig:Exp13_16}
\end{figure*}
\subsection{Environments}
\paragraph{Environment~1:}
This environment has a static field with no information decay, as the process noise is zero. As shown in Fig.~\ref{fig:Exp1_5}, the left half of the grid has higher target clarity, while in the right half its zero. The robot first moves to the high target clarity region and increases clarity until the target is reached. Once the clarity deficit in that region approaches zero, the objective provides no further incentive to remain there, allowing the robot to move freely across the grid. Since there is no decay, it does not need to revisit previously explored areas.
\paragraph{Environment~2:}
This environment has a static field with no information decay, as the process noise is zero. As shown in Fig.~\ref{fig:Exp1_5}, the right half of the grid has higher target clarity, while in the left half its zero. The robot first moves to the high target clarity region and increases clarity until the target is reached. Once the clarity deficit in that region approaches zero, the objective provides no further incentive to remain there, allowing the robot to move freely across the grid. Since there is no decay, it does not need to revisit previously explored areas.
\paragraph{Environment~3:}
This environment has a static field with no information decay, as the process noise is zero. As shown in Fig.~\ref{fig:Exp1_5}, the bottom half of the grid has higher target clarity, while in the top half its zero. The robot first moves to the bottom half and increases clarity until the clarity deficit in the region approaches zero; the objective provides no further incentive to remain there; the robot then moves freely across the grid so it explores the top half as well . Since there is no decay, it does not need to revisit previously explored areas.
\paragraph{Environment~3:}
This environment has a static field with no information decay, as the process noise is zero. As shown in Fig.~\ref{fig:Exp1_5}, the bottom half of the grid has higher target clarity, while in the top half its zero. The robot first moves to the bottom half and increases clarity until the clarity deficit in the region approaches zero; the objective provides no further incentive to remain there; the robot then moves freely across the grid so it explores the top half as well . Since there is no decay, it does not need to revisit previously explored areas.
\paragraph{Environment~4:}
This environment has a static field with no information decay, as the process noise is zero. As shown in Fig.~\ref{fig:Exp1_5}, the bottom left quadrant of the grid has a higher target clarity, while its zero elsewhere. The robot first moves to the bottom left and increases clarity until the clarity deficit approaches zero; the objective provides no further incentive to remain there; the robot then moves freely across the grid so it explores other parts as well. Since there is no decay, it does not need to revisit previously explored areas.
\paragraph{Environment~5:}
This environment has a static field with no information decay, as the process noise is zero. As shown in Fig.~\ref{fig:Exp1_5}, the top left quadrant of the grid has a higher target clarity, while its zero elsewhere. The robot first moves to the top left and increases clarity until the clarity deficit approaches zero; the objective provides no further incentive to remain there; the robot then moves freely across the grid so it explores other parts as well. Since there is no decay, it does not need to revisit previously explored areas.
\paragraph{Environment~6:}
This environment as shown in Fig.~\ref{fig:Exp6_9} has a decay field that has a high decay in the lower half and no decay in the top half. and has uniform high target clarity across the grid. The robot first moves around in the upper half and the clarity deficit approaches zero; the objective provides no further incentive to remain there, and additionally there is no decay here so it will stay that way; the robot then moves and stays in the bottom half, which has high decay.
\paragraph{Environment~7:}
This environment as shown in Fig.~\ref{fig:Exp6_9} has a decay field that has a high decay in the upper half and no decay in the lower half. and has uniform high target clarity across the grid. The robot first moves around in the lower half and the clarity deficit approaches zero; the objective provides no further incentive to remain there, and additionally there is no decay here so it will stay that way; the robot then moves and stays in the upper half, which has high decay.
\paragraph{Environment~8:}
This environment contains decaying and non decaying regions, as shown in Fig.~\ref{fig:Exp6_9}. The top right and bottom left quadrants have a target clarity of 1, while elsewhere it is 0. Decay is present only in the top right region. The robot first moves to the bottom-left quadrant, minimizing the target clarity deficit. It then moves to the top-right, where decay causes clarity to decrease over time. The robot stays in this region to minimize the clarity deficit, illustrating how the planner adapts to temporal decay.Also the robot does not spend much time in other two quadrants.
\paragraph{Environment~9:}
    This environment contains variable decaying and non decaying regions, as shown in Fig.~\ref{fig:Exp6_9}. The top half and bottom left quadrant have a target clarity of 1, while elsewhere it is 0. Decay is present in higher is top right and then lower in top left and then 0 elsewhere. The robot first moves to the bottom-left quadrant, minimizing the target clarity deficit. It then moves to the top-half, where decay causes clarity to decrease over time. The robot moves between quadrants in this region to minimize the clarity deficit.Also the robot does not spend much time in bottom right quadrant.
\paragraph{Environment~10:}
    This environment has variable decay across the environment, as shown in Fig.~\ref{fig:Exp10_12}. The target clarity is also complex and varies through the environment and is higher in some parts. The robot first moves between the regions of high target clarity, minimizing the target clarity deficit. It then moves to the areas, where decay causes clarity to decrease over time. Also the robot does not spend much time in low target clarity regions.
\paragraph{Environment~11:}
    This environment has variable decay across the environment, as shown in Fig.~\ref{fig:Exp10_12}. The target clarity is also complex patchy pattern that varies through the environment and is higher in some parts. The robot first moves among regions of high target clarity, minimizing the target clarity deficit. It then moves to the areas, where decay causes clarity to decrease over time.The robot does not spend much time in low target clarity regions.
\paragraph{Environment~12:}
    This environment has variable decay across the environment, as shown in Fig.~\ref{fig:Exp10_12}. The target clarity is also complex patchy pattern that varies through the environment and is higher in some parts. The robot first moves among regions of high target clarity, minimizing the target clarity deficit. It then moves to the areas, where decay causes clarity to decrease over time.The robot does not spend much time in low target clarity regions.
\paragraph{Environments~13,~14,~15,~16:}
    These environments include obstacles near regions of high target clarity in some of the previous environments, as shown in Fig.~\ref{fig:Exp13_16}. They evaluate the planner’s ability to safely explore informative areas separated by obstacles. The robot successfully navigates around obstacles without collisions or deadlocks, adjusting its motion based on the clarity deficit and visiting high-clarity regions as needed. The integrated gating mechanism ensures that all trajectories remain within the safe set, maintaining safety while the planner adapts its motion to minimize the clarity deficit in the presence of obstacles.

\end{document}